\documentclass[journal,final]{IEEEtran}
\usepackage{amsmath,amssymb,amsbsy,amsthm,setspace}

\let\epsilon\varepsilon
\let\phi\varphi

\def\Ilim{I_{\operatorname{lim}}}
\def\Ia{{}^aI}

\def\I{\mathbb I}
\def\N{\mathbb N}
\def\Z{\mathbb Z}

\def\cX{\mathcal X}
\def\cY{\mathcal Y}
\def\cA{\mathcal A}

\def\S{\mathcal S}
\def\F{\mathcal F}
\def\E{\mathbb E}

\def\argmin{\operatorname{argmin}}

\def\as{\text{ a.s.}}

\def\Ilim{I_{\operatorname{lim}}}

 \newtheorem*{theorem*}{Theorem}
 \newtheorem{theorem}{Theorem}
 \newtheorem{lemma}{Lemma}
\newtheorem{proposition}{Proposition}
\newtheorem{corollary}{Corollary}

\newtheorem{definition}{Definition}
\theoremstyle{definition}

\title{Unsupervised model-free representation learning}
\author{Daniil Ryabko
  \thanks{Parts of this paper were presented at conferences ISIT'13 and ALT'13; this paper contains new results and corrections.}
}
\begin{document}
\maketitle
\begin{abstract}
Numerous control and learning problems face the situation where sequences of  high-dimensional highly dependent data 
are available but no or little feedback is provided to the learner, which makes any inference rather challenging.
To address this challenge, we formulate the following problem.  Given a series of observations  $X_0,\dots,X_n$ coming from a large (high-dimensional) space $\cX$,
find a representation function $f$ mapping $\cX$ to a  finite space
$\cY$ such that  the series  $f(X_0),\dots,f(X_n)$ preserves as much information as possible  about the original time-series
dependence in $X_0,\dots,X_n$.   We show that, for stationary time series, the function $f$ can be selected as the one maximizing
a certain information criterion that we call time-series information.
Some properties of this functions are investigated, including its uniqueness and consistency of its empirical estimates.
 Implications for the   problem of optimal control are presented. 
\end{abstract}

\section{Introduction}
In many learning and control problems one has to deal with the situation where the input data 
is high-dimensional and abundant, but the feedback for the learning algorithm is scarce or absent.
In such situations, finding the right {\em representation} of the data can be the key to solving the problem.
The focus of this work is on problems in which all or a large significant part of the relevant information is 
in the time-series dependence of the process. This is the case in many applications, starting with such classical ones as speech or hand-written
text recognition, and, more generally, including control and learning problems in which the input is a stream of sensor data of
an agent interacting with its environment.  

Thus, assume first that we are given a  stationary sequence $X_0,\dots,X_n,\dots$ where $X_i$ belong to a large (continuous, high-dimensional) space $\cX$.
For the moment, assume that the problem is non-interactive (the control part is introduced later).
We are looking  for a representation $f(X_0),\dots,f(X_n),\dots$ where $f(X_i)$  belong to a small, finite space 
 $\cY$. The representation should be such as to preserve most of the information about time-series dependence in the original sequence. 
To formalize this goal, let us first consider the following ``{\em ideal}'' situation. There exists a function $f:\cX\to\cY$ such 
that  each random variable $X_i$ is {\em independent of the rest of the sequence $X_0,\dots,X_{i-1},X_{i+1},\dots,X_n,\dots$    given} $f(X_i)$  (for each $i,n\in\N$).
That is, all the time-series dependence is in the sequence  $(f(X_i))_{i\in\N}$ %
and, given this sequence, the original sequence $(X_i)_{i\in\N}$ %
can be considered redundant, in the sense that $X_i$ are conditionally independent.  In this case we say that {\em $(X_i)_{i\in\N}$ are  
conditionally independent given $(f(X_i))_{i\in\N}$.}
We can show that in this ``ideal'' situation the function $f$ maximizes the following information criterion 
\begin{equation}\label{eq:tsi}
 I_\infty(f):=h(f(X_0))- \sum_{k=1}^\infty w_k h_k (f(X)),
\end{equation}
where $h_k(f(X))$ is the order-$k$  Shannon entropy  rate  of the (stationary) time series  $(f(X_i))_{i\in\N}$ (see, e.g.\ \cite{Cover:91} or the next section for definitions), and $w_k$ 
are summable real weights: here we let $w_k:=1/(k(k+1))$. %
This means that for any other function $g:\cX\to\cY$ we have $I_\infty(f)\ge I_\infty(g)$, with equality if and only $(X_i)_{i\in\N}$ are also 
conditionally independent given $(g(X_i))_{i\in\N}$. 

This allows us to pass to the {\em non-ideal} situation, in which there is no function $f$ that satisfies the conditional independence criterion. Given 
a set of functions mapping $\cX$ to $\cY$ consider a function $f$ that maximizes~(\ref{eq:tsi}).
Such a function  can be said to preserve the most of time-series dependence of the original time series $(X_i)_{i\in\N}$ (as opposed to the ideal case, 
in which such a function $f$ preserves all of the time-series dependence).

For a given function $f$, the quantity~(\ref{eq:tsi})  can be estimated empirically. %
Moreover, one can show that, under certain conditions, it is possible to estimate~(\ref{eq:tsi}) {\em uniformly} over a set $\mathcal F$ of functions $f:\cX\to\cY$.
Importantly, the estimation can be carried out without estimating the distribution of the original time series $(X_i)_{i\in\N}$.

Of particular interest, especially in control problems, is the case where the time series $(X_i)_{i\in\N}$ form a Markov process.
In this case, in the ``ideal'' situation (when $(X_i)_{i\in\N}$ are  
conditionally independent given $(f(X_i))_{i\in\N}$) one can show that the process $(f(X_i))_{i\in\N}$
is also Markov, and $I_\infty(f)=I_1(f):= h(f(X_0))- h(f(X_1)|f(X_0))$; here   $h(f(X_1)|f(X_0))$ is the conditional Shannon entropy 
of $f(X_1)$ given $f(X_0)$.

Next, assume that at each time step $i$ we are allowed to take an {\em action} $A_i$, and the next observation $X_{i+1}$ depends not only on $X_0,\dots,X_i$ but
also on the actions $A_1,\dots,A_i$. Thus, we are considering the {\em control} problem, and the time series $(X_i)_{i\in\N}$ do not have to be stationary any more. 
In this situation, the time-series information becomes dependent on the policy $\pi$ of the learner, that is, on the way the actions are chosen. 
However, we can show that, in the Markov case, under some mild connectivity conditions, if the ideal representation function $f$ exists, then to find it it is  enough to consider just one random policy that takes
all actions with non-zero probability.
This means that one can find the representation function $f$ while executing a random policy, without any feedback from the environment (i.e., without rewards). One can 
then use this same representation to solve the target control problem more easily. Finding a representation function without access to rewards is especially useful in situations where rewards are costly to obtain. For example, one can imagine situations where the dynamics of the environment can be simulated, but the rewards cannot.

\subsection{Examples}
Next we consider some examples of problems where all or most of the information one would like to learn is contained in the time-series dependence.

\subsubsection{Example 1: type-written text}
 Let the sequence $X_i$, $i=1..N$ be a type-written text in English. Here each $X_i$ is the image of a letter printed on a typewriter, and together the sequence represents a text. The text is readable, so each image $X_i$ uniquely determines a letter; let $f:\cX\to\cY$ be this mapping, where $\cX$ is the space of images of letters and $\cY$ is the finite space of all  letters of the English language. Note that $X_i$ are conditionally independent given $f(X_i)$. The text itself is not i.i.d.\ nor Markov, but can be thought of as being stationary. Here it is easy to picture the conditional distribution of $X$ given $f(X)$: there are $|\cY|$ many mutually singular distributions, each corresponding to a letter on the typewriter. By pressing, say, the key  ``A'' one generates an image corresponding to that letter, independently of all other images of the letter ``A'' before or after it; the dependence between this image and the rest of the images in the text is only via the labels, that is, the dependence is in the text and not in the images themselves. Thus we have described an example of the {\em ideal} situation formulated above.  If instead of a type-written text we consider a hand-written one, then the conditional independence assumption does not hold, because the way one writes a letter depends on the preceding and the following letters. Thus, we are no longer in the ideal situation; we can get back to it by considering a larger space $\cY$, for example that of all words, so that $X_i$ is an image of a hand-written word and $Y_i$ is the word itself. A somewhat middle-ground formulation in which we are close to the ideal situation is to consider the space of all pairs or triplets of letters.

\subsubsection{Example 2: online chess game}
  A player plays a game of chess on his computer. After each move, he takes a screenshot of the part of the screen that contains the board. These screenshots constitute  the sequence $X_i$. This is a control problem, as the player can take actions (chess moves) $A_i$, to which the environment reacts (the other player makes moves as well). It can be assumed to be a (rewardless) Markov decision process (MDP), meaning that  (the next state) $X_{i+1}$ only depends on (the previous state) $X_i$ and action $A_i$. The function $f$ maps the space $\cX$ of screenshots to the finite space $\cY$ of all strings that describe the positions of all the figures on the chess board (considering some canonical representation of the board). Here again $X_i$ are conditionally independent given $f(X_i)$. To add some randomness to the problem (the screenshots may be thought of as deterministic), one can consider, instead of the screenshots, photographs of the computer screen. This would still preserve the conditional independence property. However, if we consider a physical version of the game, where $X_i$ are photographs of the physical board, then the conditional indepedence property breaks (or holds only approximately). Indeed, the position of each of the figures that does not make a move  remains the same, and as the figure is not placed on a very precise position on the board, there is some game-unrelated randomness that is preserved from one move to the next. Note that if someone does not know anything about the game of chess, and does not receive any feedback (not even the binary ``lose/win'' signal in the end), then there is no way of knowing that this position information is irrelevant. Since we are considering the control problem without rewards, this lack of conditional independence becomes important: we cannot throw away the time-series information (here, contained in the position of the figure inside its cell on the board) without knowing that it is irrelevant for rewards.

\subsection{Prior work}
Learning representations, feature learning, model learning, as well as model and feature selection, are different variants
and different names of the same general problem: making the data more amenable to learning.  
From the vast literature available on these problems we only mention a few that are somehow related to the approach taken in this work.
First, note that   in our ``ideal'' (conditional independence) case, if we  further assume that   $(X_i)$ form a Markov chain, 
then we get a special case of  Hidden Markov models (HMM)~\cite{Rabiner:86}, with (unobserved) $f(X_i)$ being hidden states. Indeed, as it was mentioned, in this  case $f(X_i)$ form a Markov 
chain (Section~\ref{s:mark}), and thus can be considered hidden states; the dependence between  $X_i$ and $f(X_i)$ is deterministic, as opposed to stochastic in HMM,
so we get a special case. Thus, the general case (non-ideal situation, $X_i$ are not necessarily Markov) can be considered a generalization of HMMs.
A related approach to finding representations in HMMs is that of  \cite{Hutter:10} (see also \cite{Hutter:09c}).
The setting of \cite{Hutter:10}  can be related to our setting in Sections~\ref{s:mark},~\ref{s:cont}. Specifically, \cite{Hutter:10} considers 
environments generated by HMMs in which the hidden states are deterministic functions of the observed variables. %
 The approach of~\cite{Hutter:10} is then to maximize a penalized likelihood function, where the penalty is 
for larger state spaces. Consistency results are obtained for the case of finite or countably infinite sets of maps (representation functions), which are given
by so-called finite-state machines of bounded memory one of which is the true environment.    %

From a different perspective, if $X_i$ are independent and identically distributed and, instead of the time-series dependence (which, of course, is absent in this case), we want
to preserve as much as possible  of the information about another sequence of variables (labels) $Y_1,\dots,Y_n$, then one can arrive 
at the information bottleneck method~\cite{Tishby:99}.  The information bottleneck method  can, in turn, be seen as a generalization of the rate-distortion theory of Shannon \cite{Shannon:59}.
While the classical formulation of the information bottleneck method deals with i.i.d.\ data, the concept had been applied to
dependent data as well. Thus, applied to dynamical systems,  the information bottleneck method can be formulated \cite{Creutzig:09} as follows: minimize $I(\text{past; representation}) -\beta I(\text{representation; future})$, where
$\beta$ is a parameter. A related idea is that of causal states \cite{Shalizi:01}: two histories belong to the same causal state if and only if they give the same conditional distribution over futures.
What distinguishes the  approach of this work from those described is that we never have to consider the probability distribution of the input time  series $X_i$ directly~--- only 
through the distribution of the representations $f(X_i)$. Thus, modelling or estimating $X_i$ is not required; this is particularly important for  empirical estimates.   

For the control problem, to relate the proposed approach to others, first  observe that in the
case of an MDP, in the ``ideal'' scenario, that is, in the case when there exists a function $f:\cX\to\cY$ such that $(X_i)_{i\in\N}$ are conditionally 
independent given $(f(X_i))_{i\in\N}$, then for any states $x,x'\in\cX$ for which $f(x)=f(x')$ all the transition probabilities are the same.
In other words, states $x,x'\in\cX$ for which $f(x)=f(x')$ are equivalent in a very strong sense, and the function $f$ can be viewed as state aggregation. Generalizations of this
equivalence and aggregation (in the presence of rewards or costs) are studied in the  bisimulation and homomorphism literature \cite{Givan:03,Ferns:06,Taylor:09,Ravindran:03}.
The main difference of our approach (besides the absence of rewards) is in the treatment of approximate (non-ideal) cases and in the way we propose to find the representation
(aggregation) functions. In bisimulation this is approached via a metric on the state space defined using a distance between the transition (and reward) probability distributions,
which then has to be estimated \cite{Ferns:06,Taylor:09}. %
In our approach, all that has to be estimated concerns the representations $f(X)$, rather than the observations (states) $X$ themselves.

In the context of  reinforcement learning with rewards, a related problem  is that of finding a (concise) representation 
of the input space such that the resulting process on representations is Markovian \cite{Maillard:11,Maillard:13}. 

It should also be noted that the conditional independence property  has been previously studied in a different context (classification) in \cite{Ryabko:06condiid}.
The latter work shows that if the objects $(X_i)_{i\in\N}$ are conditionally independent given the labels $(Y_i)_{i\in\N}$ then, effectively,
one can use classification methods developed to work under the assumption of i.i.d.\ object-label pairs. Thus, if one is given some examples of object-label  pairs $(X,f(X))$, then one can use both classification methods and time-series information to learn the function~$f$. In other words, \cite{Ryabko:06condiid} shows that time-series information does not hurt: the methods proven to work only when there is none (i.i.d.\ object-label pairs) are, in fact, applicable to a wider range of situations.  On the other hand, the results of the present work show that time-series information can be very useful, even when no object-label examples are given. 

\subsection{Organization}
The rest of the paper is organized as follows. Section~\ref{s:pr} introduces some notation and definitions. Section~\ref{s:main} 
introduces the model and gives the main results concerning representation functions for stationary time series. Section~\ref{s:mark} considers the special case of (stationary) Markov 
chains; Section~\ref{s:uni} 
presents results on uniform empirical approximation of time-series information. Finally,  Section~\ref{s:cont} extends the model and results to the control problem. Some longer proofs are deferred to the Appendix. %

\section{Preliminaries}\label{s:pr}
Let $(\cX,\mathcal F_\cX)$ and $(\cY,\mathcal F_\cY)$ be probability spaces. Here we assume that  $\cY$ is finite; extensions to infinite spaces $\cY$  are  left for future work. The idea is that $\cY$, as a representation space,  should be ``smaller'' than $\cX$, which, in turn, is thought of as ``large'' (e.g. high-dimensional Euclidean space).  The space $(\cX,\mathcal F_\cX)$ is assumed to be such that extensions to time-series distributions and regular conditional probabilities are well-defined. A sufficient condition for this is that $(\cX,\mathcal F_\cX)$ possesses a standard basis, e.g., is a Polish space \cite{Gray:88}. 
When speaking about conditional probability distributions, for random variables $X$ and $Y$ the notation $P_{X|Y}(X|Y)$ refers to a regular conditional probability measure (e.g., \cite[Section 5.8]{Gray:88}). The equality of any such two conditional probabilities is always understood in the almost-sure sense, that is, for almost all values of the random variables under the condition.  Moreover, we shall omit the subscript in the expressions like $P_{X|Y}(X|Y)$, i.e.\ simply use $P(X|Y)$, somewhat abusing the notation which may also refer to the probability of taking certain value, as in $P(X=x)$; however, the distinction will be clear from the context.

The upper-case letters, such as $X,Y,$ etc.\ are reserved for random variables, while the lower-case letters $x,y,$ etc.\ for their realisations.

Time-series (or process) distributions  are probability measures on the space $(\cX^\N,\mathcal F_{\N})$ of one-way infinite sequences (where $\mathcal F_{\N}$ is the Borel
 sigma-algebra of $\cX^\N$).
We use the abbreviation $X_{0..k}$ for $X_0,\dots,X_k$. 
A distribution  $\rho$ is stationary if $\rho(X_{0..k}\in A)=\rho(X_{n..n+k}\in A)$ 
for all $A\in\mathcal F_{\cX^k}$, $k,n\in\N$ (with $\mathcal F_{\cX^k}$ being the sigma-algebra of $\cX^k$).

A stationary distribution on $\cX^\N$ can be uniquely extended to a distribution on $\cX^\Z$ (that is, to a time series $\dots,X_{-1},X_0,X_1,\dots$);
we will assume such an extension whenever necessary.

The following notation is used for entropies and information; see, e.g., \cite{Cover:91} for a thorough exposition.
For a discrete random variable $Z$ denote $h(Z)$ its Shannon entropy $h(Z):=-\sum_{z}P(Z=z)\log P(Z=z)=-\E\log P(Z)$ where $0\log0:=0$. For a pair of random variables $Z_1,Z_2$, their joint entropy is simply $h(Z_1,Z_2):=-\sum_{z_1,z_2}P(Z_1=z_1,Z_2=z_2)\log P(Z_1=z_1,Z_2=z_2)$ and the conditional entropy is defined as $h(Z_2|Z_1):=-\sum_{z_1}P(Z_1=z_1)\sum_{z_2}P(Z_2=z_2|Z_1=z_1)\log P(Z_2=z_2|Z_1=z_1)=-\E\log P(Z_2|Z_1)$.
We use the notation  $h_0(f)$ for the entropy of $f(X_0)$ %
\begin{equation}\label{eq:ent}
h_0(f):=h(f(X_0)),
\end{equation}
 and $h_k(f)$ for the $k$-order entropy  rate of $(f(X_i))_{i\in\N}$
\begin{equation}\label{eq:entk}
h_k(f):= h(f(X_{k})|f(X_0),\dots,f(X_{k-1})).
\end{equation}
For stationary time series $(f(X_i))_{i\in\N}$ the (limit) entropy rate, which always exists (see, e.g.,\cite{Cover:91}) is defined as 
$$
h_\infty(f):=\lim_{k\to\infty} h_k(f).
$$
For random variables $Y_1,Y_2$, the mutual information $I(Y_1;Y_2)$ is defined as 
$
I(Y_1;Y_2):= h(Y_1)-h(Y_1|Y_2).
$
For stationary time series $Y_1,Y_2,\dots$ we have, by stationarity, the following simple identity which is used often in this text: $$I(Y_1;Y_2)= h(Y_2)-h(Y_1|Y_2).$$
\section{Time-series information for stationary distributions}\label{s:main}
This section describes the main results concerning representation functions for stationary time series. 
We first introduce the ``ideal'' situation in which $(X_i)_{i\in\N}$ are conditionally independent given $(f(X_i))_{i\in\N}$ for some function $f:\cX\to\cY$, 
and define time-series information.
We then show that under this condition the function $f$ maximizes time-series information. 
\begin{definition}[conditional independence given labels]\label{d:ciid}
 We say that $(X_i)_{i\in\N}$ are conditionally independent given $(f(X_i))_{i\in\N}$, if
 for   all $n,k,$ and all $i_1,\dots,i_k\ne n$,  $X_n$ is independent of  $X_{i_1},\dots,X_{i_k}$ given $f(X_n)$:  %
 \begin{equation}\label{eq:ciid}
    P(X_n|f(X_n),X_{i_1},\dots,X_{i_k})=P(X_n|f(X_n))\as
 \end{equation}
\end{definition}

One can think of $f(X_i)$ as a sequence of {\em representations}: these representations preserve all the information about time-series dependence that is 
present in the original sequence $(X_i)_{i\in\N}$: indeed, the latter variables become independent given the representations.

The following simple implications of Definition~\ref{d:ciid} will be used repeatedly in the proofs. 
\begin{lemma}\label{l:prop}
 Assume that $(X_i)_{i\in\N}$ are conditionally independent given $(f(X_i))_{i\in\N}$, and let $g_i,g_j'$ $(i,j\in\N)$ be any (measurable) functions from $\cX$ to $\cY$. 
Then, for all different  values of the indices, we have
\begin{subequations}
\begin{multline}\label{eq:mag}
    P(X_{i_1},\dots,X_{i_k}| X_{j_1},\dots,X_{j_m},g_l(X_l))
  \\ =P(X_{i_1},\dots,X_{i_k}| f(X_{j_1}),\dots,f(X_{j_m}),g_l(X_l))\as
\end{multline}
 \begin{multline}\label{eq:mag2} 
  P(X_{i_1},\dots,X_{i_k}|X_{j_1},f(X_{j_1}),\dots,X_{j_{m}},f(X_{j_{m}}),g_l(X_l))\\= P(X_{i_1},\dots,X_{i_k}|f(X_{j_1}),\dots,f(X_{j_{m}}),g_l(X_l))\as,
 \end{multline}
 \begin{multline}\label{eq:magic} 
  h(g_1'(X_{i_1}),\dots,g_k'(X_{i_k})
   \\|g_1(X_{j_1}),f(X_{j_1}),\dots,g_m(X_{j_{m}}),f(X_{j_{m}}),g_l(X_l))= \\h(g_1'(X_{i_1}),\dots,g_k'(X_{i_k})|f(X_{j_1}),\dots,f(X_{j_{m}}),g_l(X_l))\as,
 \end{multline}
\begin{multline}\label{eq:magic2}
  h(g_1'(X_{i_1}),\dots,g_k'(X_{i_k})|f(X_{j_1}),\dots,f(X_{j_{m}}),g_l(X_l))\\\le h(g_1'(X_{i_1}),\dots,g_k'(X_{i_k})\\|g_1(X_{j_1}),\dots,g_m(X_{j_{m}}),g_l(X_l))\as,
\end{multline}
 \begin{multline}  \label{eq:magic4}
  I(g_1'(X_{i_1}),\dots,g_k'(X_{i_k});\\ 
g_1(X_{j_1}),f(X_{j_1}),\dots,g_m(X_{j_{m}}),f(X_{j_{m}}),g_l(X_l))=\\ I(g_1'(X_{i_1}),\dots,g_k'(X_{i_k});f(X_{j_1}),\dots,f(X_{j_{m}}),g_l(X_l))\as,
 \end{multline}	
\begin{multline}\label{eq:magic3}
  I(g_1'(X_{i_1}),\dots,g_k'(X_{i_k});f(X_{j_1}),\dots,f(X_{j_{m}}),g_l(X_l))\\\ge I(g_1'(X_{i_1}),\dots,g_k'(X_{i_k});\\g_1(X_{j_1}),\dots,g_m(X_{j_{m}}),g_l(X_l))\as
\end{multline}
\end{subequations}

\end{lemma}
\begin{proof}
 The first statement is simply the symmetry of conditional independence. We show it for one index on each side, as the general case is analogous: the extra random variables (including $g_l(X_l)$ in the condition) are pure spectators. Applying the Bayes formula and the conditional independence,  w.p.~1 we have the following chain of equalities for conditional distributions:
\begin{multline*}
P(X_{i}| X_{j})=P(X_{i}| X_{j},f(X_{j}))
\\
= \frac{P(X_{j}|X_{i},f(X_{j}))P(X_{i},f(X_{j}))}{P(X_{j},f(X_{j}))}\\
=\frac{P(X_{j}|f(X_{j}))P(X_{i},f(X_{j}))}{P(X_{j}|f(X_{j}))P(f(X_j))}=P(X_{i}|f(X_{j})).
\end{multline*}

For the rest of the statements, we have \eqref{eq:mag}$\Rightarrow$\eqref{eq:mag2}$\Rightarrow$\eqref{eq:magic}$\Rightarrow$\eqref{eq:magic2}; \eqref{eq:magic}  $ \Rightarrow$\eqref{eq:magic4}; \eqref{eq:magic2}  $ \Rightarrow$\eqref{eq:magic3}.
\end{proof}

One can show that a function $f$ which gives the property of conditional independence, if it exists, is unique up to permutations and up to 0-probability sets,
provided $\cY$ is the smallest set for which such a function exists.

\begin{definition}[Minimal representation set] We say that $\cY=\{1,\dots,K\}$ is the minimal representation set if there exists a function $f:\cX\to\cY$ such that
$(X_i)_{i\in\N}$ are conditionally independent given $(f(X_i))_{i\in\N}$, and for every $K'<K$ 
 there is no function $f:\cX\to\cY':=\{1,\dots,K'\}$ such that $(X_i)_{i\in\N}$ are conditionally independent given $(f(X_i))_{i\in\N}$. 
\end{definition}

\begin{proposition}[Uniqueness of representation]\label{th:one} Let  $\cY$ be the minimal representation set and let $f,g:\cX\to\cY$ be such that $(X_i)_{i\in\N}$ are conditionally independent given either  $(f(X_i))_{i\in\N}$ or $(g(X_i))_{i\in\N}$. Then there exists a permutation $\pi:\cY\to\cY$ such that $f(X_i)=\pi(g(X_i))$ a.s. 
\end{proposition}
\begin{proof}
Here it will be convenient for us to reverse the order of random variables in the definition of conditional independence, which we can do since independence is symmetric: 
$  P(X_{i_1},\dots,X_{i_k}|f(X_n),X_n)=P(X_{i_1},\dots,X_{i_k}|f(X_n)) $ a.s., and likewise for $g$ (see~\eqref{eq:mag2}).

Let $X_{\backslash i}$ denote the random variables  $(X_t)_{t\in\N, t\ne i}$. Observe that, if the distributions $P(X_{\backslash i}|f(X_i)=a)$ and $P(X_{\backslash i}|f(X_i)=b)$ coincide for some values $a,b\in\mathcal Y$ for all $i$ and the conditions have non-zero probability at least for some $i$, then we must have $a=b$, since otherwise $\mathcal Y$ would not be the minimal representation set (we could merge $a$ and $b$). The same holds for the function $g$. Likewise, for every $a\in\mathcal Y$ we must have $P(f(X_i)=a)>0$ at least for some $i$ (since otherwise we could merge $a$ with an arbitrary letter), and similarly for $g$. 

Therefore, to prove the proposition, it is enough 
to show that for every $a\in\mathcal Y$ there exists $b\in\mathcal Y$ such that $P(X_{\backslash i}|f(X_i)=a)=P(X_{\backslash i}|g(X_i)=b)$ for all $i\in\N$ for which the condition has non-zero probability simultaneously; this relation would establish the required permutation $\pi$.
To do so, consider $a,b\in\cY$ and $A=f^{-1}(a)$, $B=g^{-1}(b)$.  Note  that for any measurable $T\subset \cX^k$ and any different values of the indexes $i,j_1,\dots,j_k$, conditional independence   implies, via~\eqref{eq:mag2},  that
$P((X_{j_1},\dots,X_{j_k})\in T|X_i=s, f(X_i)=a)$ is a.s.\ constant in $s$. Similarly, $P((X_{j_1},\dots,X_{j_k})\in T|X_i=s, g(X_i)=b)$  is a.s.\ constant in $s$.
If $P(X_i\in A\cap B)>0$ for some $i\in\N$ then these constants should coincide, which means that the distributions  $P(X_{\backslash i}|f(X_i)=a)$ and $P(X_{\backslash i}|g(X_i)=b)$ are a.s.\ equal. This  implies the statement. 
\end{proof}

The main quantity of interest is $I_\infty(f)$, which is a formalization of the notion of time-series information. It quantifies
the amount of time-series dependence in the series $(X_i)_{i\in\N}$.

\begin{definition}\label{d:iinf}
For a time series $f(X_0),\dots,f(X_n),\dots$, define its $k$th order time-series information as
\begin{equation}\label{eq:tik}
  I_k(f):= h_0(f)- h_k(f) %
\end{equation}
and its time-series information as
\begin{equation}\label{eq:ti}
  I_\infty(f):= \sum_{k=1}^\infty w_k I_k(f)=h_0(f)-\sum_{k=1}^\infty w_k  h_k(f),
\end{equation}
where we set $w_k:=1/(k(k+1))$ (however, any positive weights that sum to 1 may be chosen).
\end{definition}

For stationary time series, from this definition we immediately obtain the following identity
 \begin{equation}\label{eq:inf}
   I_k(f)=I(f(X_k);f(X_{1}),\dots,f(X_{k-1})).
 \end{equation}
The following theorem is the main result concerning representations of stationary time series. 
\begin{theorem}\label{t:m}
 Let $(X_i)_{i\in\N}$ be a stationary time series, and let $f:\cX\to\cY$ be such that $(X_i)_{i\in\N}$  are conditionally independent given $(f(X_i))_{i\in\N}$.
Then for any $g:\cX\to\cY$ we have $I_\infty(f)\ge I_\infty(g)$, with equality if and only if  $(X_i)_{i\in\N}$ are conditionally independent given $(g(X_i))_{i\in\N}$.
\end{theorem}
The proof is deferred to the appendix. %

Thus, given a set $\mathcal F$ of representation functions  $f:\cX\to\cY$, the function that is ``closest'' to 
satisfying the conditional independence property given in Definition~\ref{d:ciid} can be defined as   the one that maximizes~(\ref{eq:ti}).
If the set $\mathcal F$ is finite and the time series $(X_i)_{i\in\N}$ are stationary, then it is
 possible to find the function that maximizes~(\ref{eq:ti}) given 
a large enough sample of the time series,   without
knowing anything about its  distribution. 
Indeed, it suffices  to have a consistent estimator for  the entropy
$h_k(f)$, which can %
 be estimated simply using empirical plug-in estimates. 
In practice it is clearly not necessary to compute the infinite sum in~(\ref{eq:ti}), but only as many summands as is computationally feasible and statistically meaningful; computing $l= O(\log n)$ summands (where $n$ is the length of the time-series available) seems reasonable in view of obtaining consistent frequency estimates: using the definition of the weights $w_k=1/(k(k+1))$ we can upper-bond  the error from not computing the rest of the summands by $O(\log(|\cY|/(l+1))$.

\subsection{Alternative formulations}
Theorem~\ref{t:m} can be formulated in a slightly different way without making reference to $I_\infty$ but only to $I_k$. Such an alternative formulation makes clearer the role of $I_k$ and avoids the use of (rather arbitrary) parameters $(w_k)_{k\in\N}$ in the definition of $I_\infty$ (Definition~\ref{d:iinf}). At the same time, it does not yield a specific function to optimize in order to find the representation $f$.
\begin{theorem*}[alternative formulation of Theorem~\ref{t:m}]
 Let $(X_i)_{i\in\N}$ be a stationary time series, and let $f:\cX\to\cY$ be such that $(X_i)_{i\in\N}$  are conditionally independent given $(f(X_i))_{i\in\N}$.
Then for any $g:\cX\to\cY$ and all $k\in\N$ we have $I_k(f)\ge I_k(g)$, where the inequality is strict for at least some $k$ unless  $(X_i)_{i\in\N}$ are conditionally independent given $(g(X_i))_{i\in\N}$.
\end{theorem*}
The proof of Theorem~\ref{t:m} (given in the appendix) carries over to this formulation essentially unchanged.

An alternative way of defining $I_\infty$, which appears attractive, is to replace the sum $\sum_{k=1}^\infty  w_kh_k(f)$ in~\eqref{eq:ti} by the limit $h_\infty(f)$, that is, defining
$$
\Ilim(f):=h_0(f)-h_\infty(f).
$$
This quantity is appealing since it avoids using the weights $w_k$. 
However, a problem arises with that, unlike for $I_\infty$, the inequalities  $I_k(f)<I_k(g)$ for all $k$ cannot be used to conclude directly that
$\Ilim(f)<\Ilim(g)$. This means that one cannot obtain an analogue of Theorem~\ref{t:m} for $\Ilim$ in the same way.

\section{Time-series information for Markov chains}\label{s:mark}
For the control problem, a special role is played by Markov environments; we first look 
at the simplifications gained by making this assumption in the stationary case, before considering the control problem itself in the following.

If the series $(X_i)_{i\in\N}$ form a stationary %
 Markov process then the situation simplifies considerably.
First of all,  if  $(X_i)_{i\in\N}$ are conditionally independent given $(f(X_i))_{i\in\N}$ then $(f(X_i))_{i\in\N}$ also form a stationary %
Markov chain.
Moreover, to find the function that maximizes the time-series information~(\ref{eq:tsi}) it is enough to find the function 
that maximizes a simpler quantity %
$I_1(f)=I(f(X_0);f(X_1))$, as the following theorem shows.

\begin{theorem}\label{th:mark}
Suppose that   $X_i$ form a stationary Markov process  and   $(X_i)_{i\in\N}$ are conditionally independent given $(f(X_i))_{i\in\N}$. Then
\begin{itemize}
 \item[(i)]  $(f(X_i))_{i\in\N}$ also form a stationary Markov chain.
\item[(ii)] In this case $I_\infty(f)$ is the mutual information between $f(X_0)$ and $f(X_1)$:
 \begin{equation}\label{eq:imar}
  I_\infty(f)=I_1(f)=I(f(X_0);f(X_1)),
\end{equation}
 and  for any $g:\cX\to\cY$ we have $I_1(f)\ge I_1(g)$ with equality if and only if  $(X_i)_{i\in\N}$ are conditionally independent given $(g(X_i))_{i\in\N}$.
\end{itemize} 
\end{theorem}
\begin{proof}
We use the notation $Y_i:=f(X_i)$.
For the first statement, observe that
 \begin{multline}
 P(Y_{n+1}|Y_1\dots,Y_n)=P(Y_{n+1}|Y_1,X_1,\dots,Y_n,X_n)\\=P(Y_{n+1}|Y_n,X_n)=P(Y_{n+1}|Y_n),
\end{multline}
where we have used successively~\eqref{eq:mag2},  the Markov property for $(X_i)_{i\in\N}$ and again~\eqref{eq:mag2}. This establishes the Markov property for the process $(Y_i)_{i\in\N}$; its stationarity follows from that of $(X_i)_{i\in\N}$.  %

For the second statement, first note that  $h_k=h_1$, $k\ge1$ for Markov chains, implying~(\ref{eq:imar}). Next, for any $g:\cX\to\cY$ the process $g(X_i)$ is stationary, 
which implies $h_k(g(X))\le h_1(g(X))$, $k\ge1$. Thus, using Theorem~\ref{t:m}, we obtain
 \begin{equation*}
  I_1(f)=I_\infty(f)\ge I_\infty(g)\ge h_0(g)-h_1(g)=I_1(g).
 \end{equation*}
The statement about the case $I_1(f)=I_1(g)$ also follows from Theorem~\ref{t:m}.
\end{proof}

\section{Uniform approximation}\label{s:uni}
Given an infinite (possibly uncountable) set $\mathcal F$  of functions $f:\cX\to\cY$,
we want to find a function that maximizes $I_\infty(f)$. 
Here we first  consider the problem of approximating $I_k(f)$, and then, based on this,  proceed with the problem of approximating $I_\infty(f)$. 

Since we do not know $I_k(f)$, we can select a  function that maximizes the empirical estimate $\hat I_k(f)$. 
The question arises, under what conditions is this procedure consistent? 
The requirements we impose to obtain consistency of this procedure  are of the following  two types: first, the set $\mathcal F$ should be sufficiently 
small, and, second, the time series $(X_i)_{i\in\N}$ should be such that uniform (over $\mathcal F$) convergence
guarantees can be established.  Here the first condition is formalized in terms of VC dimension, and the second 
in terms of mixing times. We show that, under these conditions,  the empirical estimator is  indeed consistent and learning-theory-style 
finite-sample performance guarantees can be established.

\begin{definition}[Estimators]
For a function $f:\cX\to\cY$ and a sample $X_1,\dots,X_n$ define the following estimators:
$\hat p_f(y):={1\over n}\sum_{i=1}^n\I(f(X_i)=y)$, and analogously for $\hat p_f(y_1,\dots,y_k)$,
the multivariate entropies and mutual informations $\hat I_k$, the latter with plug-in estimator $\hat p$ for $p$. The dependence on $n$ is left implicit.
\end{definition}

\begin{definition}[$\beta$-mixing coefficients, e.g.,\cite{Bosq:96}]
For a process distribution $\rho$ define the mixing coefficients
$$
 \beta(\rho,k):=\sup_{\substack{A\in \sigma(X_{-\infty..0}),\\ B\in\sigma(X_{k..\infty})}} |\rho(A\cap B)-\rho(A)\rho(B)|
$$ where $\sigma(..)$ denotes the sigma-algebra of the random variables in brackets.
\end{definition}
When the limit $\lim_{k\to\infty}\beta(\rho,k)$ is $0$, the process $\rho$ is sometimes called absolutely regular; this condition is much stronger than ergodicity, but is much weaker than
the i.i.d.\ assumption.

\def\S{{\mathcal S}}
For a set of indicator functions $\mathcal F$ from $\cX$ to $\{0,1\}$ the symbol $\S(\mathcal F,n)$ is used for the $n$-th shatter coefficient of the 
set~$\mathcal F$:
$$
\S(\mathcal F,n):=\max_{\{x_1,\dots,x_n\}\subset\cX}\#\Big\{ \{i:C(x_i)=1\} :C\in\mathcal F\Big\},
$$
that is, the maximal number of different subsets of $n$ points that can be picked out by the set of indicator functions  $\mathcal F$.
The Vapnik-Chervonenkis (VC) dimension of a set $\mathcal F$ is defined as the maximal integer $d$ such that $S(\mathcal F)=2^d$; see \cite{Vapnik:98,Devroye:96}.

 The general tool that we use to obtain performance guarantees in this section is the following bound
that can be obtained from  \cite[Theorem 3]{Karandikar:02}. Let $\F$ be a set  of VC dimension~$d$ (interpreted as a set of binary functions) and let $\rho$ be a stationary distribution. Then
\begin{multline}\label{eq:mixtl}
 q_n(\rho,\F,\epsilon)
:= \rho\Big(\sup_{g\in\F} |{1\over n}\sum_{i=1}^{n}  g (X_{i})
-\E_{\rho} g (X_1)| >\epsilon\Big)\\\le n\beta(\rho,t_n)+8t_n^{d+1}e^{-l_n\epsilon^2/8},
\end{multline}
where $t_n$ is a parameter and $l_n:=n/t_n$ . %
The parameters $t_n$  should be set according to the values of $\beta$ in order to optimize the bound.

Furthermore,  assume geometric $\beta$-mixing distributions, that is, $\beta(\rho,t)\le\gamma^t$ for some $\gamma<1$.
Letting $\l_n=t_n=\sqrt{n}$ the bound~(\ref{eq:mixtl}) becomes
\begin{equation}\label{eq:mix}
 q_n(\rho,\F,\epsilon)\le n\gamma^{\sqrt{n}}+8n^{(d+1)/2}e^{-\sqrt{n}\epsilon^2/8}=:\Delta(d,\epsilon,n,\gamma).
\end{equation}

Geometric $\beta$-mixing properties can be demonstrated for large classes of (k-order) (PO)MDPs \cite{hernandez:03}, and for many other 
distributions. 

The VC-dimension and the bounds~\eqref{eq:mixtl}, \eqref{eq:mix} above concern sets $\mathcal F$ of binary-valued functions. In order to reduce the case of non-binary spaces $\cY$ to the binary case, we will consider the indicator functions ${\mathbb I}_{\{x\in\cX: g(x)=y\}}:\cX\to\{0,1\}$ that, for each $g$ and each given $y$, take the value $0$ on  $x$ if $g(x)\ne y$ and 1 otherwise. 
\begin{theorem}\label{th:mix}
Let the time series  $(X_i)_{i\in\N}$  be generated by a stationary distribution $\rho$  whose $\beta$-mixing
coefficients satisfy $\beta(\rho,t)\le\gamma^t$ for some $\gamma<1$. Let $\F$  be a set of  functions $f:\cX\to\cY$ 
such that for  each $y\in\cY$ the VC dimension of the set $\{{\mathbb I}_{\{x\in\cX: g(x)=y\}}: g\in\F\}$ is not greater than $d$. 
Then
\begin{multline}\label{eq:thmix}
P( \sup_{g\in\F} | \hat I_k(g)-I_k(g)|>\epsilon) 
\\\le2|\cY|^{k+1}\Delta(7kd,\min\{\epsilon/(6(k+1)|\cY|^{k+1}\log|\cY|),\\h^{-1}(\epsilon/(6|\cY|^{k+1}))\},n-k,\gamma), %
\end{multline}
where $h^{-1}$ stands for the inverse of the binary entropy (and is of order $h^{-1}(\epsilon)\sim\epsilon/\log(1/\epsilon))$.
\end{theorem}
The proof is deferred to the Appendix. %

We proceed to construct an estimator of $I_\infty(g)$ which is uniformly consistent over a set $\mathcal F$ of functions $g$, provided
the time series satisfies mixing conditions.
To this end, denote $\delta_k(n,\epsilon)$ the right-hand side of~(\ref{eq:thmix}). 
Assuming some monotonically non-decreasing sequence of integers $k_n$, 
 define
\begin{equation}\label{eq:hinf}
 \hat I_\infty(g):=\sum_{k=1}^{k_n}w_k\hat I_{k}(g).
\end{equation}
Observe that for each fixed $k\in\N$, $\delta_k(n,\epsilon)$ decreases exponentially fast with $\sqrt{n}$. 
Therefore, it is possible to find a non-decreasing sequence $k_n:n\in\N$
such that $\delta_{k_n}(n,\epsilon)$ decreases as $\exp^{-\Omega(\sqrt{n})}$ with $n$ (up to polynomial factors), while $k_n\to\infty$; for example, one can take $k_n:=\log n$.
\begin{theorem}Under the conditions of Theorem~\ref{th:mix}, we have, with the choice of  $k_n$ with the asymptotic behaviour as described (e.g., $k_n=\log n$),
\begin{equation}\label{eq:cor1}
 P( \sup_{g\in\F} | \hat I_\infty(g)-I_\infty(g)|>\epsilon) \le k_n\delta_{k_n}(n,\epsilon/2),
\end{equation}
 provided $n$ is large enough to satisfy  $\sum_{i>k_n} w_i < \epsilon/2$. In this case,
$$
\sup_{g\in\mathcal F}|\hat I_\infty(g)-I_\infty(g)|\to0\as
$$

 \end{theorem}
\begin{proof}
 The first statement follows from Theorem~\ref{th:mix} by using the condition  $\sum_{i>k_n} w_i < \epsilon/2$ and the union bound: 
\begin{multline*}
 P\left( \sup_{g\in\F} | \hat I_\infty(g)-I_\infty(g)|>\epsilon\right) \\ \le 
 P\left( \sup_{g\in\F} \sum_{i=1}^{k_n}  w_i|\hat I_i -I_i(g)|>\epsilon/2\right)
\\\le \sum_{i=1}^{k_n} P\left( \sup_{g\in\F}  |\hat I_i -I_i(g)|>\epsilon/2\right) \\\le k_n\delta_{k_n}(n,\epsilon/2).
\end{multline*}

 The second statement follows from the first, using a sequence of $\epsilon$ slowly decreasing with $n$ and the Borel-Cantelli lemma.
\end{proof}

\section{The active case: MDPs without rewards}\label{s:cont}
In this section we introduce learner's actions into the protocol. 
The setting is a sequential interaction between a  learner and an environment.
Given are a space of observations $\cX$ and   a space of actions $\cA$, where $\cA$ is assumed finite.
 At each time step $i\in\N$ the environment provides 
an observation $X_i$, the learner takes an action $A_i$, then the next observation $X_{i+1}$ is provided, and so on.
Each next observation $X_{i+1}$ is generated according to some (unknown) probability distribution
$P(X_{i+1}| X_0,A_0,\dots,X_i,A_i)$.
Actions are generated by a probability distribution $\pi$ that is called a {\em policy}; in general, it has the form 
$\pi(A_{i+1}| X_0,A_0,\dots,X_i,A_i, X_{i+1})$, for all $i\in\N$. 

Note that  we do not introduce  costs or rewards into consideration.
Thus, we are dealing with an unsupervised version of the common reinforcement-learning problem;  the goal is just to find 
a concise representation that preserves the dynamics of the process.

The focus in this section 
is  on time-homogeneous Markov environments, that is, on Markov Decision Processes (MDPs) without rewards. 
Thus, we assume that $X_{i+1}$  only depends on $X_i$ and $A_i$, and this dependence is constant in $i$. This means that 
$P$ can be identified with a  function $p$ from $\cX\times\cA$ to the space $\mathcal P(\cX)$ of probability distributions on $\cX$ 
\begin{multline}\label{eq:mdp}
 P(X_{i+1}\in T| X_0,A_0,\dots,X_{i-1},A_{i-1},X_i=x,A_i=a)\\=p_{x,a}(X_{i+1}\in T) \as
\end{multline} 
for all $T\in\mathcal F_{\mathcal X}$. When $x,a$ are random, e.g., $X_0,A_0$, we will use the notation $p(X_1|X_0,A_0)$ for $p_{X_0,A_0}(X_1)$, in order to make explicit the ``time'' order of the variables in the sequence $X_0,A_0,X_1,A_1,\dots$. The notation is justified if we return back to~\eqref{eq:mdp} to see that it corresponds to $P(X_{1}| X_0,A_0)$.

In the Markov case, the observations  $X_i$ are called {\em states} and the function $p$ the {\em transition probability} function.

A policy is called {\em stationary} if each action only depends on the current state;
that is,  $\pi(A_{i+1}| X_0,A_0,\dots,X_i,A_i, X_{i+1}=x)=\pi(A_{i+1}|x)$  where, for each $x\in\cX$,
$\pi(\cdot|x)$ is a distribution over $\cA$. %

\begin{definition}[Admissible MDPs, $P^\pi,\E^\pi, I^\pi_k$, etc.]
 Call an MDP {\em admissible} if any stationary policy $\pi$ has a (unique up to  sets of measure 0)
 stationary distribution  over states. Denote (any) such distribution $P^\pi$. Moreover, the notation $\E^\pi, h^\pi, I^\pi_k$, etc.\ refers to 
the stationary distribution of the policy~$\pi$. In particular, with this notation 
\begin{multline*}
  I^\pi(f(X_0),A_0;f(X_1),A_1)\\= h(f(X_0),A_0)-h(f(X_1),A_1|f(X_0),A_0)
\end{multline*}
where the actions $A_i$ are distributed according to $\pi$ and  $h(f(X_0),A_0)$ (respectively, $h(f(X_1),A_1|f(X_0),A_0)$) is the (conditional) entropy of the pair.
\end{definition}

Thus, for an MDP with transition function $p$ and a stationary policy, the following decomposition holds true:
\begin{multline}\label{eq:decomp}
  P^\pi(X_0,A_0,\dots,X_n,A_n) \\
=P^\pi(X_0)\pi(A_0|X_0)\prod_{i=1}^{n}p(X_{i}|X_{i-1},A_{i-1})\pi(A_{i}|X_{i}). %
\end{multline}

\begin{definition}[conditional independence, active case] 
For a policy $\pi$, an environment $P$ and a measurable function $f$ we say that  $(X_i)_{i\in\N}$ are
 conditionally independent given $(f(X_i))_{i\in\N}$ {\em under the policy $\pi$} if %
\begin{multline}\label{eq:ciida}
   P^\pi(X_n|f(X_n),A_n,X_{i_1},A_{i_1},\dots,X_{i_k},A_{i_k})\\=P^\pi(X_n|f(X_n))\as
\end{multline}
for all $n,k\in\N$, and all $i_1,\dots,i_k\in\N$ such that $i_j\ne n$, $j=1..k$. 
\end{definition}
This definition implies that the actions of the policy $\pi$ depend on $X_n$ only through the representations $f(X_n)$; in other words, the policy ``knows'' the representation.

\begin{lemma}\label{l:decomp}
For an admissible MDP $P$,  a stationary policy $\pi$ and a representation function $f$, 
if $X_i,i\in\N$ are conditionally independent 
given $f(X_i),i\in\N$ for the policy $\pi$, then
\begin{multline}\label{eq:magmdp}
 \pi(A_i|X_i)=\pi(A_i|f(X_i))\as\text{ and }\\
 p(X_{i+1}|X_i,A_i)=p(X_{i+1}|f(X_i),A_i)\as;
\end{multline}
in other words, the policy $\pi$ as a function $\pi(\cdot|x)$ from $x\in\cX$ to distributions over $\mathcal A$ applied to $X_i$ is $\sigma(f(X_i))$-measurable, and the function $p_{x,a}(\cdot)$ applied to $(X_i,A_i)$ is $\sigma((f(X_i),A_i))$-measurable, where $\sigma()$ stands for the $\sigma$-algebra generated by the random variable in brackets.

Moreover, $X_i,i\in\N$ are conditionally independent 
given $f(X_i),i\in\N$ for the policy $\pi$ if and only if, for all $n\in\N$,  there exist probability distributions $q_a$, $a\in\cY$ over $(\cX,\mathcal F_\cX)$ such that the following decomposition holds true
\begin{multline}\label{eq:decomp2}
  P^\pi(X_0,A_0,\dots,X_n,A_n)
\\=P^\pi(f(X_0))q_{f(X_0)}(X_0)\pi(A_0|f(X_0)) \hfill
\\\times\prod_{i=1}^{n}p(f(X_{i})|f(X_{i-1}),A_{i-1})q_{f(X_i)}(X_i)\pi(A_i|f(X_i))%
\end{multline}

\end{lemma}
\begin{proof}
 The first statement follows from the symmetry of conditional independence; %
its  derivation is analogous to that of~\eqref{eq:mag} in Lemma~\ref{l:prop}.

For the second statement, to show the  ``only if'' part (assuming conditional independence), note that
\begin{multline}
  P(X_{i+1}|f(X_i),A_i)\\
 =P(f(X_{i+1})|f(X_i),A_i)P(X_{i+1}|f(X_{i+1}),f(X_i),A_i))\\=P(f(X_{i+1})|f(X_i),A_i)P(X_{i+1}|f(X_{i+1}))\as,
\end{multline}
where we have used conditional independence in the last equation. By stationarity, the last term is constant in $i$, so we can introduce $q_a(T):=P(X_0\in T|f(X_0)=a)$ for $T\in\mathcal F_\cX$, $a\in\mathcal A$.
Now the statement follows from~\eqref{eq:decomp} and the first statement.

Conversely, assume that~\eqref{eq:decomp} holds, and let us show that $X_i$ are conditionally independent given $f(X_i),i\in\N$.  The equality 
\begin{equation}
P^\pi(X_1|f(X_1),X_0,A_0,A_1,X_2,A_2)= P^\pi(X_1|f(X_1))\as
\end{equation}
follows by expanding the conditional distribution and applying~\eqref{eq:decomp2}. The general case of~\eqref{eq:ciida} reduces to this because of stationarity and the Markov property.
\end{proof}

Observe that if, for a stationary policy $\pi$, $X_i$ are conditionally independent given $f(X_i)$, then  for the stationary time series $(X_i,A_i)_{i\in\N}$ we can say that $(X_i,A_i)_{i\in\N}$ are conditionally independent given $(f(X_i),A_i)$, $i\in\N$. 
This means that one can apply Theorem~\ref{th:mark} to these series. The result is the following statement, in which the function maximized by the function $f$ is somewhat simplified due to the additional conditional independence properties coming from~\eqref{eq:decomp2}.

\begin{corollary}\label{th:mdp}
 For an admissible MDP $P$, a  stationary policy $\pi$ and a function $f:\cX\to\cY$, if  $(X_i)_{i\in\N}$ are
 conditionally independent given $(f(X_i))_{i\in\N}$ {\em under the policy $\pi$}, then:
\begin{itemize}
 \item[(i)]  $(f(X_i))_{i\in\N}$ also form a stationary MDP (without rewards) with the policy $\pi$,
\item[(ii)] the function $f$ maximizes the following quantity
 \begin{equation}\label{eq:imarm}
  \Ia^\pi_1(f):=I^\pi(f(X_1);f(X_0),A_0),
\end{equation}
 that is,  for any $g:\cX\to\cY$ we have $\Ia^\pi_1(f)\ge \Ia^\pi_1(g)$ with equality if and only if  $(X_i)_{i\in\N}$ are conditionally independent given $(g(X_i))_{i\in\N}$.
\end{itemize} 
\end{corollary}
\begin{proof}
The first statement follows from the first statement  of Theorem~\ref{th:mark} (applied to the time series $(X_i,A_i)$, $i\in\N$). For the second statement, note that from the second statement of Theorem~\ref{th:mark} we have that $f$ maximizes  $I^\pi_1(g(X_0),A_0;g(X_1),A_1)$, with equality  $$I^\pi(g(X_0),A_0;g(X_1),A_1)=I^\pi(f(X_0),A_0;f(X_1),A_1)$$ reached for a function $g$ if and only if  $((X_i)_{i\in\N},A_i)$ are conditionally independent given $((g(X_i),A_i)_{i\in\N}$ under policy~$\pi$.
Moreover, 
\begin{multline}\label{eq:ccc}
 I^\pi(f(X_0),A_0;f(X_1),A_1)
 \\= h(f(X_0),A_0)-h(f(X_1),A_1|f(X_0),A_0)
 \\=h(f(X_0))+h(A_0|f(X_0)) \\- h(f(X_1)|f(X_0),A_0)-h(A_1|f(X_1),f(X_0),A_0)
\\=h(f(X_0))+h(A_0|f(X_0)) \\ - h(f(X_1)|f(X_0),A_0)-h(A_1|f(X_1))
\\=h(f(X_1))- h(f(X_1)|f(X_0),A_0)\\=\Ia^\pi_1(f),
\end{multline}
where the first equality is by definition and stationarity, the second is the chain rule for entropy, the third follows from~\eqref{eq:magmdp} (cf.~\eqref{eq:magic}), the fourth by stationarity and the last equality is by  Definition~\eqref{eq:imarm}.
 
Furthermore, for any other function $g:\cX\to\cY$, note that  the third equality in~\eqref{eq:ccc} becomes inequality ($\ge$), since 
$h(A_1|g(X_1),g(X_0),A_0)\le h(A_1|g(X_1))$. The statement about the equality $\Ia^\pi_1(f)=\Ia^\pi_1(g)$ follows from the corresponding statement in  Theorem~\ref{th:mark}.
\end{proof}

Thus, in the ideal situation, when there exists a function  $f$ such that   $(X_i)_{i\in\N}$ are conditionally 
independent given $(f(X_i))_{i\in\N}$, there is a (finite) hidden state space $\cY$ and the transitions depend only on the hidden state. 
The hidden states $y_i\in\cY$ are connected to the observable  states $x_i\in\cX$ via the representation function $f$. 
The question is how to find this function $f$. The problem is that the policy (or rather, one of the policies) with which the conditional independence is achieved has to depend on $X_i$ only through $f(X_i)$, which we do not know. 
To avoid this problem, we can simply use a policy that does not depend on anything, i.e.\ a random policy. The resulting process is stationary, and we already know how to find the representation function for a stationary process. What remains to show is that the representation function does not depend on the policy, so the representation function that we would find executing a random policy is the same one we are looking for.

\begin{definition}[Random policies, connected MDPs]
Call a stationary policy $\pi$ {\em random} if $\pi(a|x)$ does not depend on $x$ and $\pi(a)>\alpha>0$ for every  $a\in\cA$.

Furthermore, call an admissible MDP {\em (weakly) connected} if there exists  a stationary policy $\pi$ such that
(equivalently:
for every random policy $\pi$)  for any other stationary policy $\pi'$ we have $P^\pi \gg P^{\pi'}$, that is, for any measurable $S\subset\cX\times\cA$,  
 $P^{\pi'}(X_0\in S)>0$ implies $P^\pi(X_0\in S)>0$. In such a case, the policy $\pi$ is called {\em exploring}. %
\end{definition}

For discrete MDPs this definition coincides with the usual definition of weak connectedness (for any pair 
of states $s_1,s_2$ there is a policy that gets from $s_1$ to $s_2$ in a finite number of steps with non-zero probability).

\begin{proposition}\label{th:mdp2} Fix an admissible %
MDP and a random policy $\pi_0$. If, for some exploring policy $\pi$ and a representation function $f$, $X_i$ are conditionally independent given $f(X_i)$ with the policy $\pi$, then also $X_i$ are conditionally independent given $f(X_i)$ with the policy $\pi_0$.
\end{proposition}
 \begin{proof}
For the random policy $\pi_0$,  we obtain
\begin{multline}\label{eq:pr}
   P^{\pi_0}(X_0,A_0,\dots,X_n,A_n)
\\=P^{\pi_0}(X_0)\pi(A_0|X_0)\prod_{i=1}^{n}p(X_{i}|X_{i-1},A_{i-1})\pi(A_i|X_i) \\
=P^{\pi_0}(X_0)\pi_0(A_0)\prod_{i=1}^{n}p(X_{i}|X_{i-1},A_{i-1})\pi_0(A_i) \as%
\end{multline}
where the first equality is from~\eqref{eq:decomp} and the second uses the independence of actions under $\pi_0$. %
Moreover, under the policy  $\pi$ we have from~\eqref{eq:magmdp}
\begin{equation}\label{eq:ppp}
 p(X_{i+1}|X_i,A_i)=q_{f(X_{i+1})}(X_{i+1})p(f(X_{i+1}|f(X_i),A_i)\as
\end{equation}
Note, however,  that the transition function $p$ in the last equation  does not  depend on the policy, that is, it is the same for  $\pi$ and $\pi_0$. 
The only thing that depends on the policy is the ``almost sure'' assertion in the end. In~\eqref{eq:ppp} it is with respect to the policy $\pi$ and, to continue~\eqref{eq:pr} we need it with respect to $\pi_0$. However, since the policy $\pi$ is exploring, the distribution $P^\pi$ dominates $P^{\pi_0}$ (by definition), so that~\eqref{eq:ppp} holds a.s.\ with respect to $\pi_0$ as well.
Therefore, we can continue~\eqref{eq:pr} to obtain
\begin{multline}
   P^{\pi_0}(X_0,A_0,\dots,X_n,A_n) \\
 =P^{\pi_0}(X_0)\pi_0(A_0)\prod_{i=1}^{n}p(X_{i}|f(X_{i-1}),A_{i-1})\pi_0(A_i)
\\=P^{\pi_0}(f(X_0))q_{f(X_{0})}(X_{0})\pi_0(A_0)
\\\prod_{i=1}^{n}p(f(X_{i})|f(X_{i-1}),A_{i-1})q_{f(X_{i})}(X_{i})\pi_0(A_i),
\end{multline}
which has the form~\eqref{eq:decomp2} for $\pi_0$. It remains to apply Lemma~\ref{l:decomp} (second statement) to conclude that $X_i$ are conditionally independent given $f(X_i), i\in\N$ under $\pi_0$.
\end{proof}

\begin{corollary}\label{th:mdpc}
 Fix an admissible weakly connected MDP $P$ and a random policy $\pi_0$. Suppose that there are no redundant actions in the set $A$, i.e.\ there are no two actions $a_1,a_2$ such that the distributions $p_{x,a_1}()$, $p_{x,a_2}()$ are the same for almost all $x\in\cX$.  Furthermore, assume that there exists an exploring policy $\pi$ and a function $f$ such that $X_i$ are conditionally independent given $f(X_i)$ under the policy $\pi$, and $\cY$ is the minimal representation set. Then $f=\argmin_f \Ia_1^{\pi_0}(f)$, that is, the representation that minimizes $\Ia$ for the random policy $\pi_0$ is the same as the one for $\pi$, and such function $f$ is unique (up to a change of notation).
\end{corollary}
\begin{proof}
 By Proposition~\ref{th:mdp2}, $X_i$ are conditionally independent given $f(X_i)$ under the random policy $\pi_0$. Therefore, we can say the same for the stationary time series $(X_i,A_i)$: they are conditionally independent given $(f(X_i),A_i)$.  Since $\cY$ is the minimal representation set and no actions are redundant, $\cY$ is also the minimal representation set for $(X_i,A_i)$ under $\pi_0$. Therefore, the representation function $f$ is unique up to permutations. Since this holds  for both policies  $\pi$ and $\pi_0$, the statement follows.
\end{proof}

Thus, we get the following recipe for finding a representation for an MDP: minimize $\Ia^{\pi_0}(f)$ while executing a random policy $\pi_0$.  If there is a representation function $f$ such that $X_i$ are conditionally independent given $f(X_i)$ with that function under some exploring policy $\pi$, then this is the function we will find executing the random policy $\pi_0$; an additional requirement that we have to impose for this to hold is that no actions are redundant.

Note that the requirement that the policy $\pi$, for which conditional independence holds, is exploring, may appear  rather strong, but it is necessary. Indeed, otherwise, there may be some parts of the space $\cX$ on which conditional independence does not hold (for any representation function), but the random policy will take us there. However, this condition is actually not that strong, since we only require that such an exploratory  policy exists; it is not a requirement on an ``optimal'' policy, since there are no rewards in this setting.

\section{Discussion and directions for future work}
 The big question  addressed in this work is: what makes for a good representation function?
This question is very general and rather challenging, and, to the author's knowledge, is a new one in the context of dependent processes.  We have  argued that a good measure of the quality of a representation function is the time-series information. The main argument for this is that, in the ``ideal'' situation, the representation function that gives conditional independence  maximizes  time-series information.   The next question is how to find a representation function that maximizes this quantity. Section~\ref{s:uni} shows that, under some conditions, it is enough to maximize the empirical time-series information. 
To understand this result better, a helpful analogy is with the  problem of classification (e.g., \cite{Vapnik:98,Devroye:96}). In the latter problem, it is intuitively clear that a good measure of  quality of a classifier is its expected error; so, unlike for representation functions, the question of what is a good classifier is easy to answer. In order to minimize the expected error, one can show that, under some conditions, it is enough to minimize the empirical error.  Note that the convergence of the empirical error does not imply the convergence of classifiers themselves. But this is not necessary, since all one cares for is the expected error, and the difference between the empirical and expected error can be bounded. 

Similarly, if one agrees that the time-series information is an adequate measure to evaluate the quality of a representation function, then it does not matter whether representation functions that minimize the empirical version converge in some (other) sense.
This said, in view of Proposition~\ref{th:one} (uniqueness of representation),  it may be possible to establish some convergence of representation functions that minimize $\hat I_k$ (at least under some conditions). This question is left for future work.

Among the most interesting directions for future work are the implications of the results presented here for different learning problems. In particular, for the control problem it would be interesting to see what happens when one adds an (arbitrary) reward function once the (approximate) representation function has been found; specifically, how the approximation error from learning the representation propagates.
For the classification problem, if one adds some labels $Y_i$ to the dependent sequence of objects $X_i$ studied  in this work, then both classification methods and the representation learning method presented can be used to address the same problem: find the unobserved labels. It would be interesting to study how these methods can be combined to complement each other. As a first step, \cite{Ryabko:06condiid} shows that, at least, in the ``ideal'' situation time-series information does not hurt and classification methods developed to work under the i.i.d.\ assumption may be used.

Another direction for future work concerns generalizations. The first interesting generalization is to continuous spaces $\cY$. More broadly, it would be interesting to study whether similar results can be obtained for data structures more general than time series, such as (infinite) graphs. Towards this end, one can note that a number of results on stationary time series generalize to stationary infinite random graphs, as shown in \cite{Ryabko:17gratest} using the formalism of \cite{Benjamini:12}.
\subsection*{Acknowledgments}
The author is grateful to the anonymous reviewer of the journal version for numerous useful suggestions and important corrections.
\section*{Appendix}%
\begin{proof}[Proof of Theorem~\ref{t:m}]
Consider the  following entropies and  information
$h_0(f,g)$,   $h_k(f,g)$, $I_k(f,g)$ and $I_\infty(f,g)$, 
 defined (in a straightforward manner)	 for the vector-valued function $(f(\cdot),g(\cdot))$ with components $f$ and $g$.
We will first show that 
\begin{equation}\label{eq:ik}
 I_k(f,g)=I_k(f)\text{  and }I_\infty(f,g)=I_\infty(f). 
\end{equation}
The latter equality follows from the former and the definition of $I_\infty$.
Introduce the short-hand notation $Y_i:=f(X_i), Z_i:=g(X_i)$, $i\in\N$.
First note that
\begin{equation}\label{eq:h0fg}
 h_0(f,g)=h(Y_0)+h(Z_0|Y_0).
\end{equation}
 Moreover, 
\begin{multline}\label{eq:h1fg}
h_k(f,g)= 
h(Y_0,Z_0|Y_{-k..-1},Z_{-k..-1})
\\=h(Y_0|Y_{-k..-1},Z_{-k..-1})+h(Z_0|Y_0,Y_{-k..-1},Z_{-k..-1})
\\
= h(Y_0|Y_{-k..-1})+h(Z_0|Y_0)%
\end{multline}
where the first equality is by definition, the second is the chain rule for entropy and  the third follows from~(\ref{eq:magic}) and 
 conditional independence of $X_i$ given $f(X_i)$. 
Thus, from~(\ref{eq:h0fg}), (\ref{eq:h1fg}) and the definition of $I_k(f)$ we get
\begin{multline*}
I_k(f,g)=h_0(f,g)- h_k(f,g)\\=h(Y_0)+h(Z_0|Y_0)- h(Y_0|Y_{-k..-1})-h(Z_0|Y_0)=I_k(f)
\end{multline*}
finishing the proof of~(\ref{eq:ik}).

From~\eqref{eq:inf}, noting that removing random variables does not increase information, we have 
\begin{equation}\label{eq:ik2}
 I_k(f,g)\ge I_k(g), 
\end{equation}
so that using~\eqref{eq:ik} we have 
\begin{equation}\label{eq:ik3}
 I_k(f)\ge I_k(g)\text{ and }I_\infty(f)\ge I_\infty(g). 
\end{equation}

To prove the theorem, it remains to show that, if $(X_i)_{i\in\N}$ are not conditionally independent given $(g(X_i))_{i\in\N}$,  then $I_\infty(f)> I_\infty(g)$.
Since we already have~\eqref{eq:ik3}, 
 it  is enough to show that the inequality %
\begin{equation}\label{eq:go2}
 I_k(f)> I_k(g)
\end{equation}
holds for some $k$; in fact, we will show that it holds from some $k$ on. 

Assume that  $(X_i)_{i\in\N}$ are not conditionally independent given $(g(X_i))_{i\in\N}$, so that
\begin{equation}\label{eq:ne2}
 P(X_{n}| g(X_{n}),X_{i_1},\dots,X_{i_l})\ne P(X_{n}| g(X_{n}))
\end{equation}
for some $l,n$ and $i_1,\dots,i_l\ne n$ on a positive-measure set.
Adding extra variables if necessary and using stationarity, we can rewrite
\begin{equation}\label{eq:ne22}
 P(X_{0}| g(X_{0}),X_{1},\dots,X_{k},X_{-1},\dots,X_{-k})\\\ne P(X_{0}| g(X_{0}))
\end{equation}
for some $k\in\N$  on a positive-measure set.
Note that if~(\ref{eq:ne22}) holds for $k\in\N$ then it also holds  for all $k'>k$.

Recall the notation $Y_i:=f(X_i)$, $Z_i:=g(X_i)$. %
Using this notation and the symmetry of conditional independence, %
 from (\ref{eq:ne22}) we obtain, on a set of positive measure,
\begin{equation}\label{eq:ne223}
 P(X_{1..k},X_{-k..-1}| X_{0}, Z_0)\ne   P(X_{1..k},X_{-k..-1}|  Z_0).
\end{equation}
Moreover,  applying the  Bayes rule, and using conditional independence  (as in the derivation of \eqref{eq:mag}) ,  we can conclude from~\eqref{eq:ne223}
that (on a set of positive measure) 
\begin{equation}\label{eq:ne2233}
 P(Y_{1..k},Y_{-k..-1}| X_{0}, Z_0)\ne   P(Y_{1..k},Y_{-k..-1}|  Z_0);
\end{equation}
the precise argument is as follows: assume the contrary and derive
\begin{multline*}
  P(X_{1..k},X_{-k..-1}| X_{0}, Z_0)
\\= \frac{P(X_{0}, Z_0|X_{1..k},X_{-k..-1})P(X_{1..k},X_{-k..-1})}{P(X_{0}, Z_0)}
\\
=  \frac{P(X_{0}, Z_0|Y_{1..k},Y_{-k..-1})P(X_{1..k},X_{-k..-1})}{P(X_{0}, Z_0)} 
\\= 
 \frac{P(Y_{1..k},Y_{-k..-1}|X_{0}, Z_0)P(X_{1..k},X_{-k..-1})}{P(Y_{1..k},Y_{-k..-1})}
\\
= \frac{P(Y_{1..k},Y_{-k..-1}| Z_0)P(X_{1..k},X_{-k..-1})}{P(Y_{1..k},Y_{-k..-1})} 
\\= 
\frac{P(Z_0|Y_{1..k},Y_{-k..-1})P(X_{1..k},X_{-k..-1})}{P(Z_0)}
\\=\frac{P(Z_0|X_{1..k},X_{-k..-1})P(X_{1..k},X_{-k..-1})}{P(Z_0)}
\\ = P(X_{1..k},X_{-k..-1}|Z_0),
\end{multline*}
where we have used the Bayes rule, \eqref{eq:mag}, again the Bayes rule, the assumption we are trying to disprove, then again the Bayes rule and \eqref{eq:mag}; arriving at a contradiction with~\eqref{eq:ne223} and establishing~\eqref{eq:ne2233}.

Since $Z_0$ is a function of $X_0$, from~\eqref{eq:ne2233} we obtain
\begin{multline}\label{eq:h2233}
 h(Y_{1..k},Y_{-k..-1}| X_{0})
= h(Y_{1..k},Y_{-k..-1}| X_{0}, Z_0)
\\ <   h(Y_{1..k},Y_{-k..-1}|  Z_0), 
\end{multline}
which, together with~\eqref{eq:magic}, implies
\begin{equation}\label{eq:h2234}
 h(Y_{1..k},Y_{-k..-1}| Y_0) <   h(Y_{1..k},Y_{-k..-1}|  Z_0).
\end{equation}
Using the chain rule for entropy, we obtain 
\begin{multline}\label{eq:h2235}
 h(Y_{-k..-1}| Y_0)+h(Y_{1..k}| Y_{-k..-1}, Y_0) 
\\< h(Y_{-k..-1}| Z_0)+h(Y_{1..k}| Y_{-k..-1}, Z_0),
\end{multline}
so that  at least one of the following two inequalities must hold 
\begin{equation}\label{eq:h2235-1}
 h(Y_{-k..-1}| Y_0)< h(Y_{-k..-1}| Z_0)
\end{equation}
or
\begin{equation}\label{eq:h2235-2}
 h(Y_{1..k}| Y_{-k..-1}, Y_0) <h(Y_{1..k}| Y_{-k..-1}, Z_0).
\end{equation}
Assume~\eqref{eq:h2235-1}; then
\begin{multline}\label{eq:der}
 I_k(f)=I(Y_0;Y_{-k..-1})=I(Y_{-k..-1};Y_0)
\\=h(Y_{-k..-1})-h(Y_{-k..-1}|Y_0) 
\\ > h(Y_{-k..-1}) - h(Y_{-k..-1}| Z_0)
 \\=I(Y_{-k..-1};Z_0)= I(Z_0;Y_{-k..-1})
\\\ge I(Z_0;Z_{-k..-1})=I_k(Z),
\end{multline}
where the last inequality uses~\eqref{eq:magic3}.
Next, assume~\eqref{eq:h2235-2}; then, using the chain rule for entropy, we obtain
\begin{multline*}
 \sum_{t=1}^kh(Y_t| Y_{-k..-1}, Y_0,Y_{1..t-1})\\ <\sum_{t=1}^kh(Y_t| Y_{-k..-1}, Z_0,Y_{1..t-1}),
\end{multline*}
 so that there must exist $t\in\{1..k\}$ such that
\begin{equation}\label{eq:h2235-22}
 h(Y_t| Y_{-k..-1}, Y_0,Y_{1..t-1}) < h(Y_t| Y_{-k..-1}, Z_0,Y_{1..t-1}).
\end{equation}
Take such a $t$; then  we derive (similarly to how \eqref{eq:der} was derived)
\begin{multline*}
I_{k+t}(f)=I(Y_t;Y_{-k..-1},Y_0,Y_{1..t-1})
\\=h(Y_t)-h(Y_t|Y_{-k..-1},Y_0,Y_{1..t-1})
\\ > h(Y_t)-h(Y_t|Y_{-k..-1},Z_0,Y_{1..t-1}) 
\\= I(Y_t;Y_{-k..-1},Z_0,Y_{1..t-1}) \\
\ge I(Y_t;Z_{-k..-1},Z_0,Z_{1..t-1}) \\ \ge  I(Z_t;Z_{-k..-1},Z_0,Z_{1..t-1})= I_{k+t}(g),
\end{multline*}
where the first inequality is from~\eqref{eq:h2235-22} and the last two inequalities follow from~\eqref{eq:magic3}.

Thus, in either case, $I_k(f)>I_k(g)$ for some $k$ (and from some $k$ on), proving the statement.
\end{proof}

\begin{proof}[Proof of Theorem~\ref{th:mix}]
Introduce the shorthand notation 
$$p_g(y_{0..k}):=P(g(X_0)=y_0,\dots,g(X_k)=y_k).$$
Define  the total variation distance between $p_g$ and its empirical estimate $\hat p_g$ as  $\alpha_g:=\sum_{y_i\in\cY,i=0..k}|p_g(y_{0..k})-\hat p_g(y_{0..k})|$.
Observe that, from the definition of mixing,  if a process $\rho$ generating  $X_0,X_1,X_2,\dots$ is mixing with coefficients $\beta(\rho,m)$
then the process made of tuples  $(X_0,\dots,X_k),(X_1,\dots,X_{k+1}),\dots$ is mixing with coefficients $\beta(\rho,m-k)$.
Next, for the VC dimensions, observe that  if, for every fixed $y\in\cY$,  the set 
$$\{\{x: g(x)=y\}:g\in\F\},$$ considered as a set of indicator functions (recalling the explanation preceding the theorem formulation), has VC dimension bounded by
 $d$  then the set 
 $$
    \{ \{ (x_1,\dots,x_k):   g_i(x_i)=y_i,i=1..k+1\} : (g_1,\dots,g_k)\in\F^k \}
$$  
has VC dimension bounded by $7kd$ (for all $(y_1,\dots,y_k)\in\cY^k$); see \cite{Vaart:09}, which also 
gives a more precise bound.
Thus, from~\eqref{eq:mix} we obtain 
\begin{equation}\label{eq:supalf}
 P(\sup_{g\in\mathcal F} \alpha_g >\epsilon)\le|\cY|^{k+1}\Delta(7kd,\epsilon/|\cY|^{k+1},n-k,\gamma).
\end{equation}

We will use the following bound from \cite[Theorem 2]{Zhang:07} that relates the difference between mutual information  to the total variation between the 
corresponding distributions of two pairs of random variables:
\begin{multline}\label{eq:wt}
 \Big|I(g(X_0),\dots,g(X_{k-1});g(X_k)) \\ -\hat I(g(X_0),\dots,g(X_{k-1});g(X_k))\Big|
\\
\le 3(k+1)\alpha_g\log|\cY|+3h(\alpha_g),
\end{multline}
where $h$ stands for the binary entropy.
Thus, 
\begin{multline}
 P\Big(\sup_{g\in\mathcal F}|I(g(X_0),\dots,g(X_{k-1});g(X_k))
  \\-\hat I(g(X_0),\dots,g(X_{k-1});g(X_k))|>\epsilon\Big) \\
 \le P\left(\sup_{g\in\mathcal F}\alpha_g\ge\epsilon/(6(k+1)\log|\cY|)\right)
 \\+P\left(\sup_{g\in\mathcal F}h(\alpha_g)>\epsilon/6\right)
\\
 \le|\cY|^{k+1}\Delta\left(7kd,\epsilon/(6(k+1)|\cY|^{k+1}\log|\cY|),n-k,\gamma\right) \\+  |\cY|^{k+1}\Delta(7kd,h^{-1}(\epsilon/6)/|\cY|^{k+1}),n-k,\gamma),
\end{multline}
where in the first inequality we used~\eqref{eq:wt} and in the second \eqref{eq:supalf}  for each summand (inverting the binary entropy for the second one) and the fact that $h$ is monotone increasing on $[0,1/2]$. The statement of the theorem follows.
\end{proof}

%
%
%
%
%
%
%
%
%
%
%
%
%
%
%
%
%
%
%
%
%
%
%
%
%
%
%
%
%
%
%
%
%

%

%
%
%


\end{document}